\pdfoutput=1

\documentclass[11pt]{article}

\usepackage[]{acl}

\usepackage{times}
\usepackage{latexsym}
\usepackage{amsmath,amsthm,amssymb}

\usepackage{listings}

\usepackage{algorithm}
\usepackage{algpseudocode}

\usepackage[utf8]{inputenc} 
\usepackage[T1]{fontenc}    
\usepackage{hyperref}       
\usepackage{url}            
\usepackage{booktabs}       
\usepackage{amsfonts}       
\usepackage{nicefrac}       
\usepackage{microtype}      
\usepackage{xcolor}         
\usepackage{graphicx}
\usepackage{hyperref}

\newtheorem{theorem}{Theorem}[section]
\newtheorem{definition}[theorem]{Definition}

\title{Know the Ropes: A Heuristic Strategy for LLM-based Multi-Agent System Design}

\author{
Zhenkun Li$^{1}$\thanks{Equal contribution as first authors.}\thanks{Corresponding authors.} \\ 
University of South Florida \\
\texttt{zhenkun@usf.edu} \And
Lingyao Li$^{*\dagger}$ \\ 
University of South Florida \\
\texttt{lingyaol@usf.edu} \AND
Shuhang Lin \\ 
Rutgers University \\
\texttt{shuhang.lin@rutgers.edu} \And
Yongfeng Zhang \\ 
Rutgers University \\
\texttt{yongfeng.zhang@rutgers.edu}\\
}

\begin{document}

\maketitle

\begin{abstract}

Single‑agent LLMs hit hard limits—finite context, role overload, brittle domain transfer. Conventional multi‑agent fixes soften those edges yet expose fresh pains: ill‑posed decompositions, fuzzy contracts, and verification overhead that blunts the gains. We therefore present Know‑The‑Ropes (KtR), a framework that converts domain priors into an algorithmic blueprint hierarchy: tasks are recursively split into typed, controller‑mediated subtasks, each solved zero‑shot or with the lightest viable boost (chain‑of‑thought, micro‑tune, self‑check). Grounded in the No‑Free‑Lunch theorem, KtR trades the chase for a universal prompt for disciplined decomposition. On a Knapsack benchmark (3–8 items) three GPT‑4o‑mini agents raise accuracy from 3\% zero‑shot to 95\% on size‑5 instances after patching a single bottleneck agent. On the tougher Task‑Assignment suite (6–15 jobs) a six‑agent o3‑mini blueprint hits 100\% up to size 10 and $\geq $84\% on sizes 13–15, versus $\leq$11\% zero‑shot. Algorithm‑aware decomposition plus targeted augmentation thus turns modest models into reliable collaborators—no ever‑larger monoliths required.
\end{abstract}

\section{Introduction}

Individual large language model (LLM) agents typically excel in the specific domain they are optimized for \cite{thirunavukarasu2023large, kasneci2023chatgpt, wu2023bloomberggpt}, yet they struggle to achieve universal versatile \cite{zhang2024chain,xu2024hallucination}. A natural antidote is division of labor—splitting tasks into specialized agents that negotiate a joint answer—and early frameworks like Mixture-of-Agents indeed show that a well-orchestrated team can outperform its best member \cite{wang2024mixture, guo2024large, bo2024reflective}.

Yet large-scale audits have cooled that optimism. Firstly, each problem needs a well-designed prompt with significant manual effort. Even though some multi-agent frameworks seems to achieve satisfying results, when evaluation leakage and prompt over-fitting are removed, the headline boosts of naïve agent swarms slip to single-digit percentages—and can even turn negative once a task needs more than two or three coordination rounds \cite{pan2025multiagent,zhu2025multiagentbench}. Post-mortems trace the shortfall to a recurring pattern: ill-posed decompositions propagate ambiguity, loose role definitions create blind spots or duplication, verification layers either invoke brittle chain-of-thought heuristics or blow up the token budget, and every extra message compounds latency and cost almost quadratically with the number of rounds \cite{ye2025task,shu2024towards}. In short, simply throwing more “brains’’ at a problem does not guarantee progress; sustainable gains demand disciplined, principled systems engineering—the very gap our work aims to close.

Know-The-Ropes (KtR) converts domain priors into an algorithmic hierarchy: recursively split a task until each leaf fits the base LLM’s zero-shot reach or the lightest augmentation (chain-of-thought \cite{wei2022chain}, small fine-tune, self-check loops, etc). Typed I/O contracts, enforced by a lightweight controller, isolate agents and prevent cross-talk, context bloat, or silent overwrites. Multi-agent system(MAS) construction thus resembles classic pipeline tuning—pinpoint the bottleneck, refine the split, and attach the cheapest fix that works.

Theory justifies the approach. The No-Free-Lunch theorem \cite{wolpert1997no,wolpert2021important} dictates that no single agent or orchestration rule prevails across all problem distributions—there is no silver bullet. Robustness must therefore flow from exploiting domain structure, not from ever-larger prompt templates. KtR meets this mandate by re-using classical algorithms whose behavior is already well understood and optimized.

Our empirical preview includes the following:

$\bullet$ {\bf Knapsack Problem (proof of concept with a lightweight base model).}
On 3–8-item instances GPT-4o-mini scores only 60 \% $\to$ 0 \% zero-shot, and a blunt task-level fine-tune barely helps. Applying KtR yields a three-agent blueprint; fine-tuning just the “trimmer’’ on 1 200 examples lifts accuracy to 95 \%–70 \%, showing that KtR can turn a compact, low-capacity model into a high-performing multi-agent system.

$\bullet$ {\bf Task-Assignment Problem (proof of scalability with a stronger base model).}
With o3-mini, a six-agent KtR blueprint tackles sizes 6–15. Splitting one bottleneck agent into two finer leaves drives those leaves to 100 \% and 97 \% accuracy and raises overall system accuracy to $\geq$ 84 \% across all sizes—demonstrating that KtR’s gains grow with base-model capacity and that the knapsack results are not an isolated success.

Our contributions lie in two aspects. 

$\bullet$ {\bf KtR framework}. By formalizing blueprint hierarchies, we apply NFL-grounded algorithmic design that bypasses documented MAS pitfalls, streamlines the assembly pipeline, and relieves performance bottlenecks.

$\bullet$ {\bf Empirical validation}. Across two canonical optimization problems, KtR transforms modest base models into systems that match—or exceed—their fine-tuned counterparts while using orders-of-magnitude less specialized data.






\section{Related Work}

Multi-Agent Systems (MAS) have been widely employed to enhance the capabilities of LLMs to tackle complex tasks \cite{qiu2024llm, yan2024opencity, ma2024coevolving, lin-etal-2024-battleagent, hua2023war, yu2024aipatient}. This is because MAS typically distribute tasks across agents that collaborate to achieve a common goal, thereby improving both efficiency and adaptability. Recent frameworks like CAMEL \cite{li2023camel} enable role-based cooperative dialogues by assigning agents distinct personas, while AutoGen \cite{wu2023autogen} and MetaGPT \cite{hong2023metagpt} orchestrate multi-role agent teams through structured conversation loops and predefined workflows. In math optimization, OR-LLM-Agent can translate natural‐language problem descriptions into formal Gurobi models---achieving an 85\% correct‐solution rate on real‐world benchmarks \cite{zhang2025or}.

Despite this excitement, studies show that simply scaling up to LLM-based MAS often yields only marginal gains over single-agent baselines \cite{pan2025multiagent}. LLM agents still struggle with context management and consistency, meaning that elaborate multi-agent prompts can fail to realize the intended collaboration \cite{bo2024reflective}. A recent systematic audit of popular MAS frameworks has identified 14 distinct failure modes \cite{cemri2025multi}, which can be grouped into three categories, including flawed design (e.g., ambiguous role definition), inter-agent misalignment (e.g., communication failures), and quality control (e.g., no reliable check mechanism).


To address these challenges, researchers have proposed multiple strategies to make LLM-based MAS more reliable \cite{zhu2025multiagentbench, tran2025multi}. A key strategy is improving the agent interaction structure \cite{zhu2025multiagentbench}. For example, the AgentDropout framework proposes a dynamic agent‐pruning strategy, which seeks to discard less critical actors during training \cite{wang2025agentdropout}. Another effective strategy is incorporating feedback and verification loops \cite{hong2023metagpt}. A recent study shows that frameworks with strong role specialization and iterative feedback mechanisms tend to outperform those without these features \cite{anonymous2025code}. In addition, systematic evaluations suggest that the communication topology matters: a well-designed protocol between agents can significantly improve collective performance on complex tasks \cite{zhu2025multiagentbench}.


While existing multi-agent frameworks and strategies have demonstrated notable progress \cite{li2025knowledge, hong2023metagpt, zhu2025multiagentbench}, they still face challenges in dynamic role reallocation and efficient inter-agent communication. These limitations become especially pronounced when addressing complex tasks, such as NP-hard optimization problems. To bridge this gap, our work proposes a heuristic strategy that embeds domain-specific rules and algorithms directly into agent coordination. This approach enables on-the-fly role adaptation and task decomposition, particularly in math optimization contexts where conventional MAS frameworks often struggle.





\section{Methodology}

\subsection{Framework Design---Know the Ropes}

We propose the heuristic framework ``Know the Ropes.'' This framework offers a structured methodology for designing specialized MAS leveraging LLMs. This heuristic focuses on translating known, effective procedures or algorithms into a coherent multi-agent architecture. As presented in Figure \ref{fig:illustration}, the core idea is to decompose a complex overall task into its fundamental computational stages. Each stage is then mapped to a well-formulated sub-task, designed to be tractable for an individual agent. These specialized agents are subsequently orchestrated to mirror the data/control flow of the original procedure, which can effectively embed problem-solving logic into the multi-agent system. The following definitions formalize the components of this framework.

\begin{figure*}[htbp]
    \centering
    \includegraphics[width=0.86\textwidth]{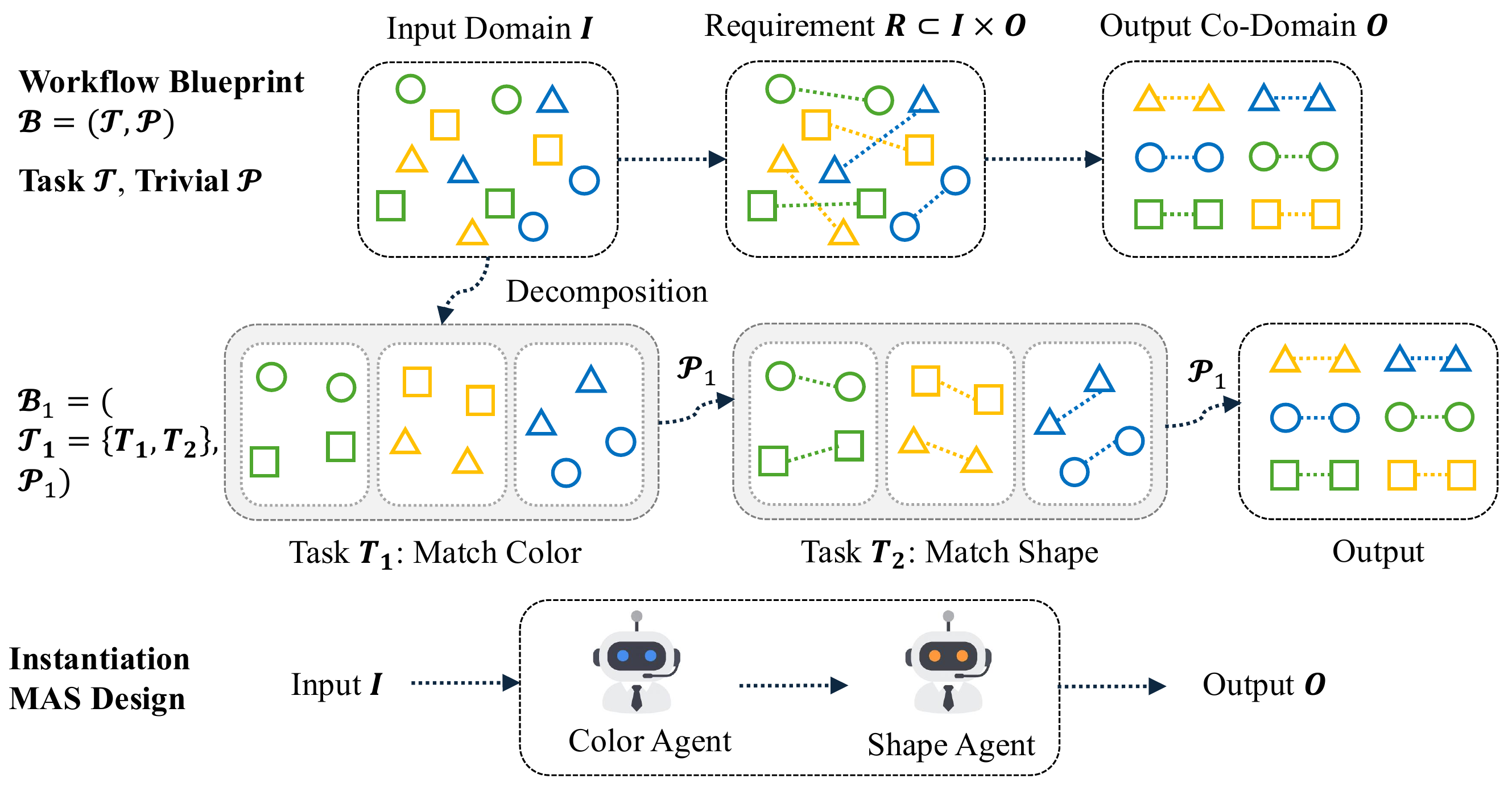}
    \caption{Illustration of the Know-The-Ropes (KtR) strategy: heuristic, prior-guided decomposition of a complex task into sub-tasks, each instantiated as a coordinated LLM agent within a multi-agent architecture.}
    \label{fig:illustration}
    \vspace{-0.6em}
\end{figure*}

\begin{definition}
	A {\bf well-formulated task} is a tuple $T = (I, O, R)$, consisting of the following: 
    
$\bullet$ Input domain ${I}$: an unambiguous description of all admissible inputs

$\bullet$ Output co-domain ${O}$ an unambiguous description of all admissible outputs

$\bullet$ Requirement relation ${R}\subset {I}\times {O}$: a relation such that for each input $x\in{I}$ it defines explicitly the subset $R(x)\subset O$ as the set of outputs that are considered correct.
\end{definition}

\begin{definition}
	A {\bf workflow blueprint} $\mathcal{B} = (\mathcal{T},\mathcal{P})$ consisting of the following.

$\bullet$ A finite set of well-formulated tasks 
$\mathcal{T} = \{T_1,\cdots, T_n\}$.

$\bullet$ An orchestration protocol $\mathcal{P}$ that specifies:

\quad$-$ The control-flow graph that determines when each $T_i$ is invoked.

\quad$-$ The data-dependency edges that map outputs of some tasks to inputs of others.

\quad$-$ Any global invariants, error-handling rules, or communication channels required to realize the end-to-end objective of $\mathcal{B}$.

\end{definition}

\begin{definition}
	Given a workflow blueprint $\mathcal{B} = (\mathcal{T},\mathcal{P})$, a {\bf decomposition} selects a task $T\in \mathcal{T}$ and replace it with a sub-blueprint $\mathcal{B}_T=(\mathcal{T}_T,\mathcal{P}_T)$ such that the following hold.

$\bullet$ Each task $T'\in \mathcal{T}_T$ is strictly simpler than $T$.

$\bullet$ The composite protocol $\mathcal{P}'$, obtained by embedding $\mathcal{P}_T$ in place of $T$ inside $\mathcal{P}$, preserves all external interface of $T$.

The result of the decomposition is a new blueprint $\mathcal{S}' = (\mathcal{T}',\mathcal{P}')$, where $\mathcal{T}'=(\mathcal{T}\backslash \{T\})\cup\mathcal{T}_T$. We record this process as
\[
\mathcal{B}\stackrel{D}{\leadsto}\mathcal{B}'
\]
\end{definition}

\begin{definition}\label{defn: M-tractable}
	Let $\mathcal{M}$ be a set of LLM models. A well-formulated task $T$ is said to be {\bf $\mathcal{M}$-tractable} if a model inside $\mathcal{M}$ satisfies the requirement relation $R_T$ with high, empirically verified accuracy, after optional augmentations (e.g., chain-of-thought prompting, tool calls, self-reflection loops, or fine-tuning).
\end{definition}

\begin{definition}
	Given a set of LLM models $\mathcal{M}$ and a blueprint $\mathcal{B}$, an {\bf $\mathcal{M}$-tractable hierarchy} is a sequence of decompositions
	\[
	\mathcal{B}\stackrel{D_1}{\leadsto}\mathcal{B}_1\stackrel{D_2}{\leadsto}\cdots\stackrel{D_n}{\leadsto}\mathcal{B}_n
	\]
	such that each task in the terminal blueprint $\mathcal{B}_n$ is $\mathcal{M}$-tractable in the sense of Definition \ref{defn: M-tractable}.
\end{definition}

\begin{definition}
	Given a set of LLM models $\mathcal{M}$ and an $\mathcal{M}$-tractable blueprint $\mathcal{B} = (\mathcal{T},\mathcal{P})$, a {\bf system instantiation} is to instantiate $\mathcal{B}$ into a MAS in the following way.

$\bullet$ We create one agent $A_i$ per tractable task $T_i\in \mathcal{T}$, bundling any necessary augmentations with the agent.

$\bullet$ We implement the orchestration protocol $\mathcal{P}$ as message-passing or function calls among agents, preserving data-dependencies and control flow.

\end{definition}

\begin{algorithm}[htbp]
\small
  \caption{Know-The-Ropes (KtR) Pseudo-code}
  \label{alg:ktr}

  \begin{algorithmic}[1]          
    \Procedure{KnowTheRopes}{$T, \mathcal{M}$}
      \State $B \gets \textsc{CreateBlueprint}(\{T\}, \textit{trivial\_protocol})$
      \While{$\exists$ $U \in B.\text{tasks}$ \& $\neg\textsc{MTractable}(U, \mathcal{M})$}
        \State $U^{\ast} \gets \textsc{ChooseTaskToDecompose}(U)$
        \State $B_{\text{sub}} \gets \textsc{DesignSubBlueprint}(U^{\ast})$
        \State $B \gets \textsc{EmbedSubBlueprint}(B, U^{\ast}, B_{\text{sub}})$
      \EndWhile

      \State $\text{MAS} \gets \textsc{InstantiateSystem}()$
      \ForAll{$V \in B.\text{tasks}$}
        \State $\textit{aug}   \gets \textsc{SelectAugmentations}(V,\mathcal{M})$
        \State $\textit{agent} \gets \textsc{CreateAgent}(V,\mathcal{M},\textit{aug})$
        \State $\text{MAS}.\text{AddAgent}(\textit{agent})$
      \EndFor
      \State \textsc{ImplementProtocol}$(\text{MAS}, B.\text{protocol})$
      \State \Return MAS
    \EndProcedure
  \end{algorithmic}
\end{algorithm}

{\bf Our method.} For a given (well-formulated) task $T$ and a given set of LLM models $\mathcal{M}$ that will be used to solve the task, we perform the following. (See Algorithm \ref{alg:ktr} and Figure \ref{fig:illustration})

$\bullet$ Define the initial blueprint $\mathcal{B} = (\mathcal{T},\mathcal{P})$, where $\mathcal{T}=\{T\}$ and $\mathcal{P}$ is trivial.

$\bullet$ Guided by domain heuristics, prior knowledge, and experiments on specific tasks, we construct an $\mathcal{M}$-tractable hierarchy
	\[
	\mathcal{B}\stackrel{D_1}{\leadsto}\mathcal{B}_1\stackrel{D_2}{\leadsto}\cdots\stackrel{D_n}{\leadsto}\mathcal{B}_n
	\]
    
$\bullet$ Materialize the terminal blueprint $\mathcal{B}_n$ as a MAS $\mathbf{MAS}(\mathcal{B}_n,\mathcal{M})$ that targets the initial task $T$.

This three-step procedure—atomic blueprint, tractable hierarchy construction, and system instantiation—provides a principled pathway from an arbitrarily complex task to a deployable multi-agent solution whose correctness hinges on model capabilities that have been explicitly validated. To demonstrate the practical application and efficacy of the ``Know the Ropes'' framework in creating such specialized MAS, we present two case studies. In each case, a well-understood algorithmic solution to a complex problem is decomposed into an M-tractable hierarchy and instantiated as a MAS.





\section{Experiment Design}

\subsection{Proof-of-Concept: 0/1 Knapsack Problem (KSP)}
To furnish a clear proof-of-concept for KtR, we start with the classical NP-hard Knapsack Problem (KSP)—a staple in resource allocation, logistics, and investment planning. By deliberately using the lightweight, general-purpose GPT-4o-mini as every agent’s backbone, we establish a modest baseline that lets us highlight how KtR’s multi-agent choreography amplifies a small model’s capability well beyond its solo limits.

Below we only present a mathematical formulation of the problem, while a more detailed explanation of the problem as well as the algorithmic solution can be found in Appendix \ref{app: KSP and TAP}.

\subsubsection{Problem Formulation}
For a Knapsack problem of size N, its input involves a weight vector $\vec{w}=(w_1,\cdots,w_N)$, a value vector $\vec{v}=(v_1,\cdots,v_N)$, and a capacity $W$. To formulate the goal, we introduce the state vector $\vec{x}=(x_1,\cdots,x_N)\in \{0,1\}^N$, i.e., all its entries take value in $\{0,1\}$. Then the objective of the problem is to find the following value
\[
Z = \max_{\substack{\vec{x}\in \{0,1\}^N \\ \vec{x}\cdot \vec{w}\leq W}} \vec{x}\cdot \vec{v}.
\]

\subsubsection{KtR Multi-Agent Design}
Following KtR heuristic, the iterative dynamic programming solution for the KSP as in Appendix \ref{sec: KSP solution} is decomposed into tasks for three specialized agents, as presented below. Prompts designed each individual agent are attached in Appendix \ref{app: prompts}.

\textbf{System Controller}: Controller initialize a set of feasible states $S_0 = \{(0,0)\}$ and controls a look on $k$ from $1$ to $N$. For each $k$, it sends $S_{k-1}$ and $(w_k,v_k)$ to Worker Agent, and then send the result plus $W$ to the Trimmer Agent. The Controller then take the union of the output of Trimmer Agent and $S_{k-1}$ to obtain $S_k$. Once all items are processed, the Controller invokes the Reporter Agent for final result. 

\textbf{Worker Agent}: This agent is responsible for the iterative state expansion by the following formula:
\[
S_{add} = \{(w+w_k,v+v_k)~{\rm for~all~}(w,v)\in S_{k-1}\}.
\]

\textbf{Trimmer Agent}: This agent performs the trimming task according to the following formula:
\[
S_{trimmed} = \{(w,v)\in S_{add}~|~w\leq W\}.
\]

\textbf{Reporter Agent}: This agent executes the solution report. It find the element with maximal value within the final state set $S_N$.

\subsection{Proof of scalability---Task Assignment Problem (TAP)}
Building on the previous section—where KtR already stretched the capabilities of the compact GPT-4o-mini on the Knapsack baseline—we now test KtR’s scalability. We upgrade the backbone to the larger o3-mini and tackle the more demanding Task-Assignment Problem (TAP), demonstrating that the framework’s performance rises in lockstep with the underlying model’s capacity.

Again, below we only present a mathematical formulation and details are referred to Appendix \ref{app: KSP and TAP}.

\subsubsection{Problem Formulation}
For a Task assignment problem of size N, its input involves an $N\times N$ cost matrix $C = (C_{ij})_{N\times N}$. We use $\mathfrak{S}_N$ to denote the set of permutations of n elements, or equivalently, the set of bijections from the set $\{1,2,\cdots, N\}$ to itself. The objective of the problem is to find the following value:
\[
Z = \max_{\sigma\in\mathfrak{S}_N}\left(\sum_{i=1}^NC_{i\sigma(i)}\right).
\]

\subsubsection{KtR Multi-Agent Design}\label{sec: TAP agent design}
Algorithm from Appendix \ref{sec: TAP problem solution} now maps to a MAS under our KtR methodology. As explained in Section \ref{sec: TAP multi-agent performance}, based on test results of the agentic tasks and heuristic argument, we further decompose the tasks from the original system design to further improve the performance. The final system design contains six agents described below. Let $N$ be the size of the problem and $C$ be the original cost matrix.

\textbf{Row Reducer}: This agent reduces rows of the matrix $C$ to obtain $C'$.

\textbf{Column Reducer}: This agent further reduces the columns of $C'$ to obtain $C''$.

\textbf{Matcher}: This agent finds a maximal collection of zeroes in the reduced matrix $C''$, such that no two zeroes share same row or column. Let $L$ be the number of zeroes in the maximal collection. 

\textbf{Painter}: When $L<N$, with input from Mather, Painter is prompted for find a minimal collection of rows and columns covering all zeroes.

\textbf{Normalizer}: Receiving input from Painter, Normalizer creates more zeroes outside of the selected rows and columns to obtain an updated matrix $C'''$.

\textbf{Reporter}: When $L=N$, Report sums up values of entries of the original cost matrix $C$ corresponding to the maximal collection of zeroes found by Matcher, and report this sum as the final solution to the problem.

The \textbf{System Controller} arranges task for Row Reducer and Column Reducer linearly, then controls a loop: Matcher first find a set of zeroes and the Controller checks if the number of zeroes $L$ equals the problem size $N$, the size of the problem. If so, the loop is broken and the positions of zeroes is sent to the Reporter for final output. Otherwise, Painters are called in to find a minimal collections of lines to cover the zeroes and Normalizer follows to create more zeroes. Then the Controller iterates the loop and send the updated matrix to Matcher.

\section{Experiment Result}

Our experimental protocol unfolds in two stages. First, we run a uniform benchmark across a suite of baseline models—including several GPT and Llama variants—to fix a reference point for each task. The second stage then splits by objective: For KSP we deliver a proof of concept, while for TAP we provide a proof of scalability. For ground truth, we use python code to randomly generate problems, and then use the Google OR-Tools \cite{ortools} as in Appendix \ref{app: ground-truth} to generate solutions to compare with.

\subsection{Experiment Result for KSP}

\subsubsection{Baseline LLM Performance}

Figure \ref{fig:ksp_baseline} shows the baseline LLM performance across multiple difficulty levels. The accuracy across difficulty levels (from 3 to 8 items) in the KSP scenario reveals substantial performance variation among the tested LLMs. Among those, the GPT-o3-mini, as a reasoning model, consistently demonstrates superior accuracy. Model GPT-4.1 outperforming its smaller counterparts, namely GPT-4.1-mini and its primer GPT-4o-mini. Other LLMs, including Claude-haiku, Llama, and Qwen series also show performance degradation, with higher variability particularly evident at greater difficulty levels. Meanwhile, the performance of final KtR multi-agent system is also drawn in Figure \ref{fig:ksp_baseline}. The comparison shows that {KtR substantially boosts performance, validating its effectiveness.}

\begin{figure}[htbp]
    \centering
    \includegraphics[width=0.5\textwidth]{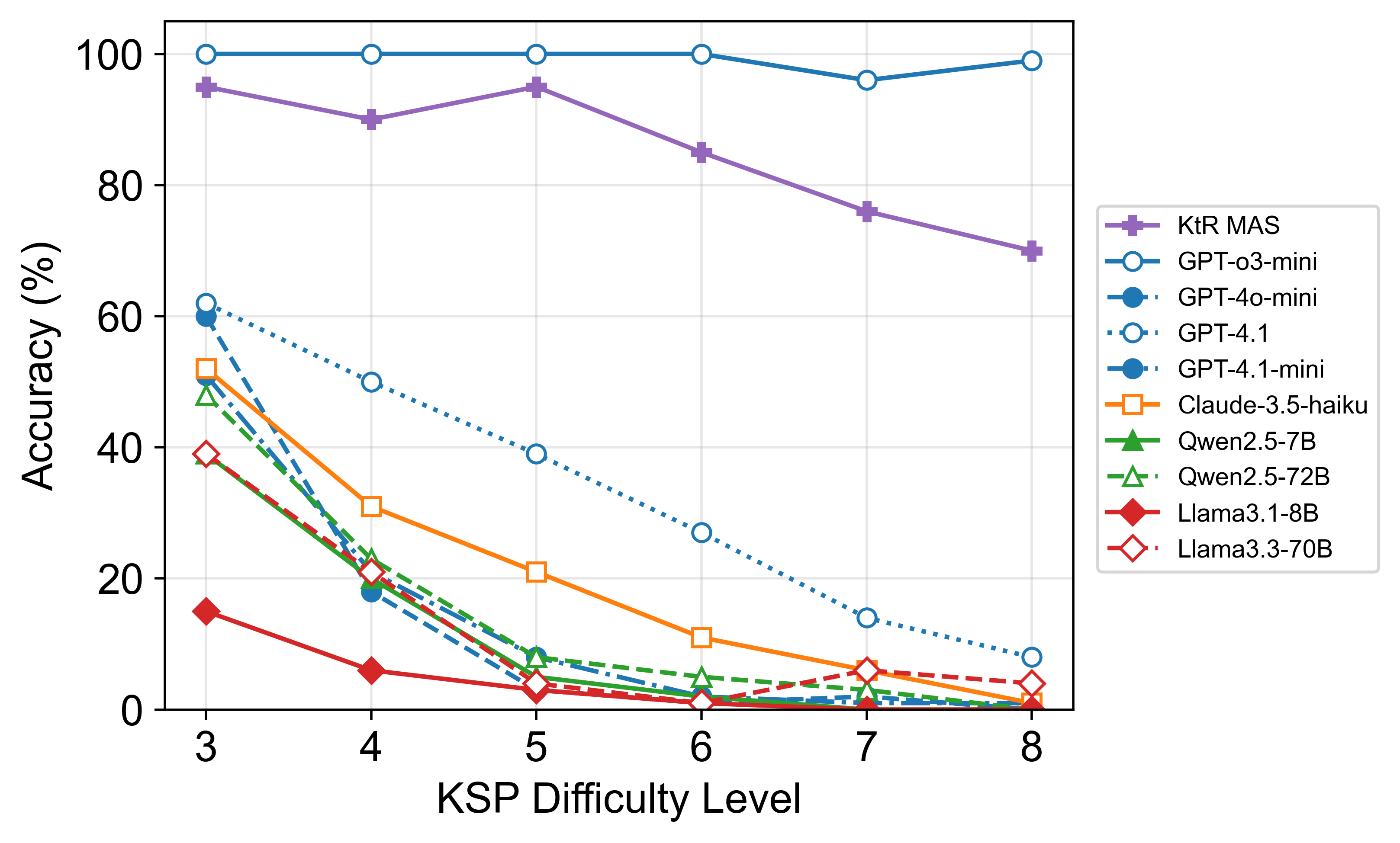}
    \vspace{-2em}
    \caption{KSP baseline performance from single LLMs as well as the KtR mulit-agent system.}
    \label{fig:ksp_baseline}
    \vspace{-1em}
\end{figure}

\subsubsection{Multi-Agent Performance}

Based on Figure~\ref{fig:ksp_baseline}, GPT-4o-mini exhibits a pronounced performance decline beginning at instances of 4 items, underscoring its limited scalability to more complex scenarios; thus, we select it as the baseline model for our KtR framework design. Figure~\ref{fig:ksp_mas} further illustrates the resulting multi-agent system along with the experimental outcomes derived from our proposed strategy.

\begin{figure*}[htbp]
    \centering
    \includegraphics[width=1\textwidth]{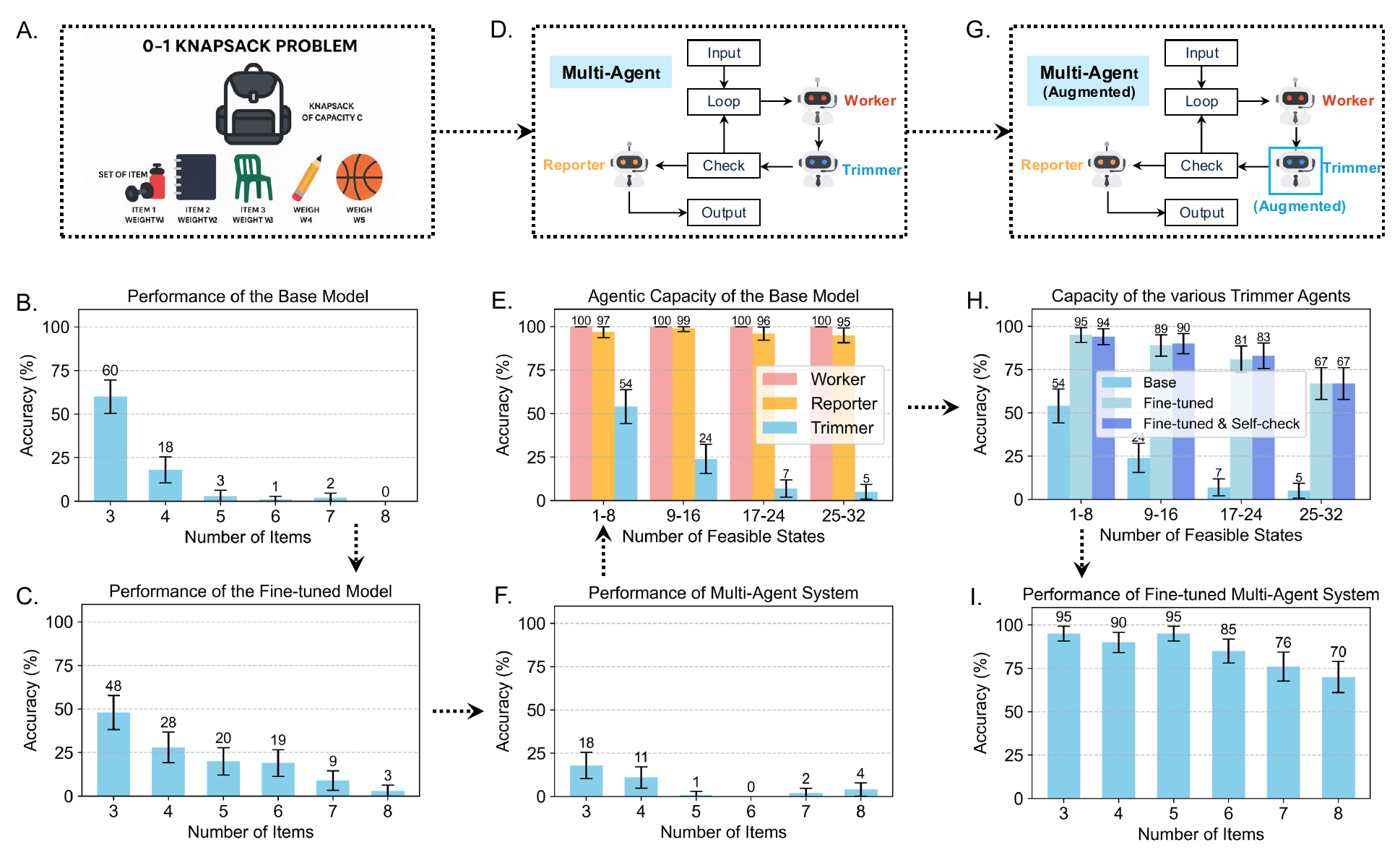}
    \vspace{-1em}
    \caption{KSP evaluation of the KtR stategy.
{\bf B}: Zero-shot accuracy of the baseline model.
{\bf C}: Zero-shot accuracy after a light, task-specific fine-tune of the same model.
{\bf D} \& {\bf G}: Blueprints of the MAS without ({\bf D}) and with ({\bf G}) augmentations.
{\bf E}: Per-agent accuracies before augmentation, revealing the system’s bottleneck.
{\bf H}: Boost delivered by two targeted augmentations—task-level fine-tuning and self-check prompting—applied to the bottleneck agent.
{\bf F} \& {\bf I}: Corresponding test accuracies for the two blueprints.
}
    \label{fig:ksp_mas}
    \vspace{-0.6em}
\end{figure*}

\textbf{Single LLM performance}. We establish two baseline performances for GPT-4o-mini acting as a single agent to solve the KSP. First, the zero-shot GPT-4o-mini is directly prompted with KSP instances. As Figure \ref{fig:ksp_mas}B s shows, its accuracy decreases from 60\% for 3 items to 0\% (8 items). Second, we evaluate a fine-tuned GPT-4o-mini (standalone). Figure \ref{fig:ksp_mas}C indicates that fine-tuning offers some improvement over the zero-shot, but still as low as 3\% for 8-item KSP.


\textbf{Standard MAS performance}. Following our KtR heuristic, we map the algorithm for KSP into a MAS design, illustrated in Figure \ref{fig:ksp_mas}F. Initially, each agent is driven by the standard, non-fine-tuned GPT-4o-mini. The performance of this standard MAS is presented in Figure \ref{fig:ksp_mas}F. Its performance descreases from 18\% for 3 items to 4\% for 8 items. This initial result implies that without augmenting the agents' abilities, the MAS does not effectively handle the task.

We profile each agent in isolation (Figure \ref{fig:ksp_mas}E) and uncover a single choke point: Trimmer. Its accuracy collapses as the feasible-state set $S_k$ (cf. Section \ref{sec: KSP solution}) grows—54 \% for 1–8 states, 24 \% for 9–16, 7 \% for 17–24, and just 5 \% for 25–32. Because the algorithm loops once per state, even small per-iteration errors compound, and this cascading inaccuracy ultimately sinks the entire run.

\textbf{Augmented MAS performance}. To eliminate the bottleneck, we fine-tune the Trimmer’s GPT-4o-mini backbone (Figure \ref{fig:ksp_mas}G, highlighted as 'Augmented Trimmer'). Accuracy leapt to 95 \% for 1–8 feasible states, 89 \% for 9–16, 81 \% for 17–24, and 67 \% for 25–32 (Figure \ref{fig:ksp_mas}H). Adding a lightweight self-check—prompting the model to audit its own answer—preserved or marginally improved these gains. Replacing the bottleneck with the fine-tuned Trimmer lifts end-to-end KSP accuracy to near-saturation across sizes (Figure \ref{fig:ksp_mas}I): 95 \% for 3-item instances, 90 \% for 4, 95 \% for 5, 85 \% for 6, 76 \% for 7, and 70 \% for 8. A single targeted upgrade thus turns KtR into a consistently high-performing solver as the problem scales.

\subsection{Experiment result on TAP}

\subsubsection{Baseline LLM Performance}

\begin{figure}[htbp]
    \centering
    \includegraphics[width=0.45\textwidth]{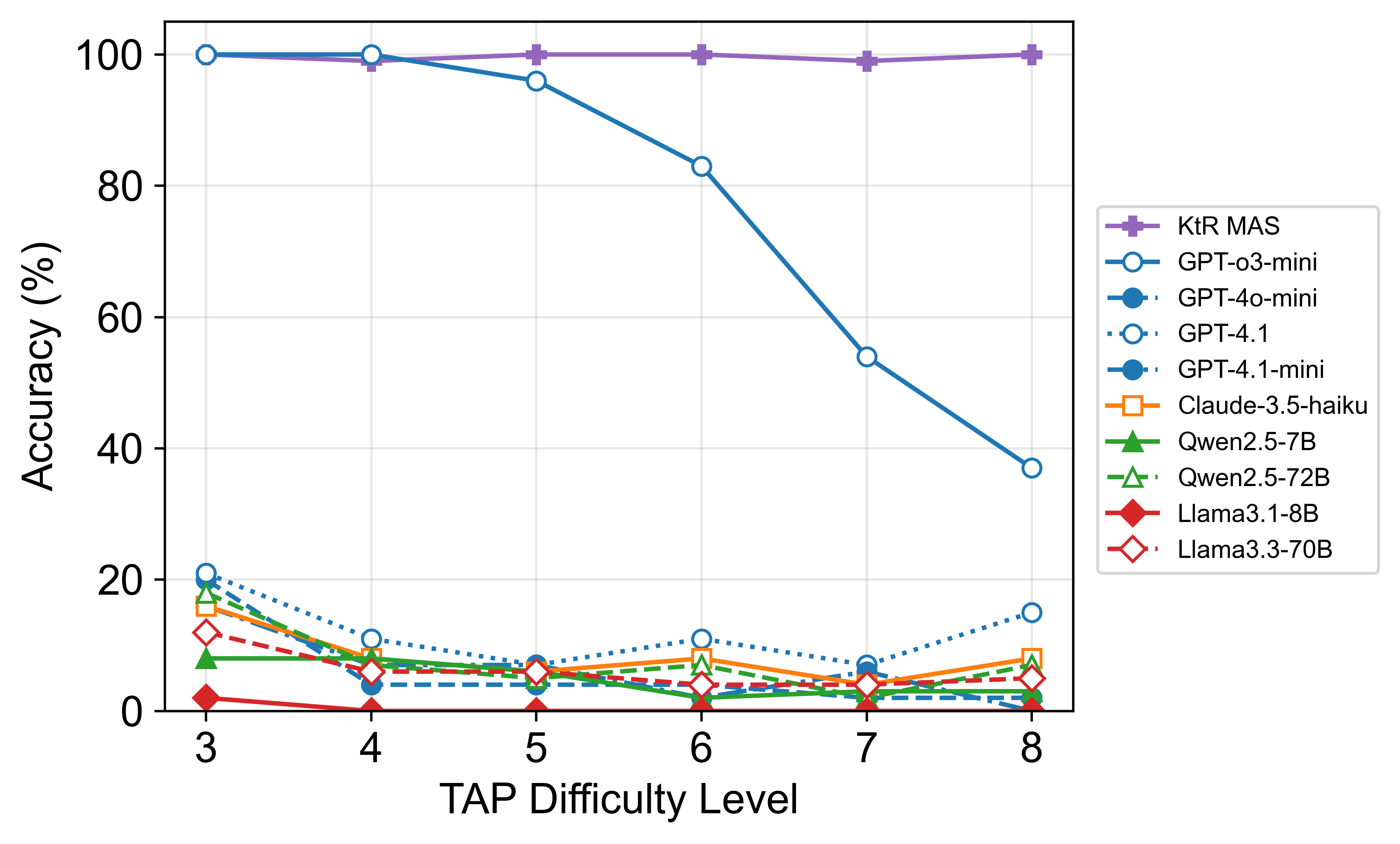}
    \vspace{-1em}
    \caption{TAP baseline performance from single LLMs as well as the KtR mulit-agent system.}
    \label{fig:tap_baseline}
    \vspace{-1em}
\end{figure}

Figure~\ref{fig:tap_baseline} illustrates the baseline performance of multiple LLMs on the TAP task across multiple difficulty levels (from 3 to 8 tasks). The results reveal marked differences in model capabilities. The only reasoning model, GPT-4o-mini, consistently outperforms all others, exhibiting strong accuracy at lower difficulty levels, though its performance declines as task complexity increases. In contrast, GPT-4.1 demonstrates moderate but stable accuracy across all difficulty levels, surpassing its mini-sized counterparts. Other models, including Claude-3.5-Haiku, Qwen2.5, and Llma-3 variants, show intermediate performance with variability.

We observe that {single-agent models (e.g., GPT-3-mini, GPT-4-mini, GPT-4.1) drop to 30-50\% accuracy at TAP levels 7-8, while KtR MAS maintains steady performance near 100\%, even surpassing reasoning models, demonstrating its exceptional robustness and generalization capabilities.}
\subsubsection{Multi-Agent Performance}\label{sec: TAP multi-agent performance}

Based on Figure~\ref{fig:tap_baseline}, GPT-o3-mini consistently outperforms other LLMs across all evaluated tasks, making it our choice for subsequent experiments. Our goal is to assess the scalability of our proposed strategy and investigate how its performance evolves as task difficulty increases. Figure~\ref{fig:tap_mas} illustrates the MAS design and corresponding experimental outcomes obtained using our heuristic-based approach.

\begin{figure*}[htbp]
    \centering
    \includegraphics[width=1\textwidth]{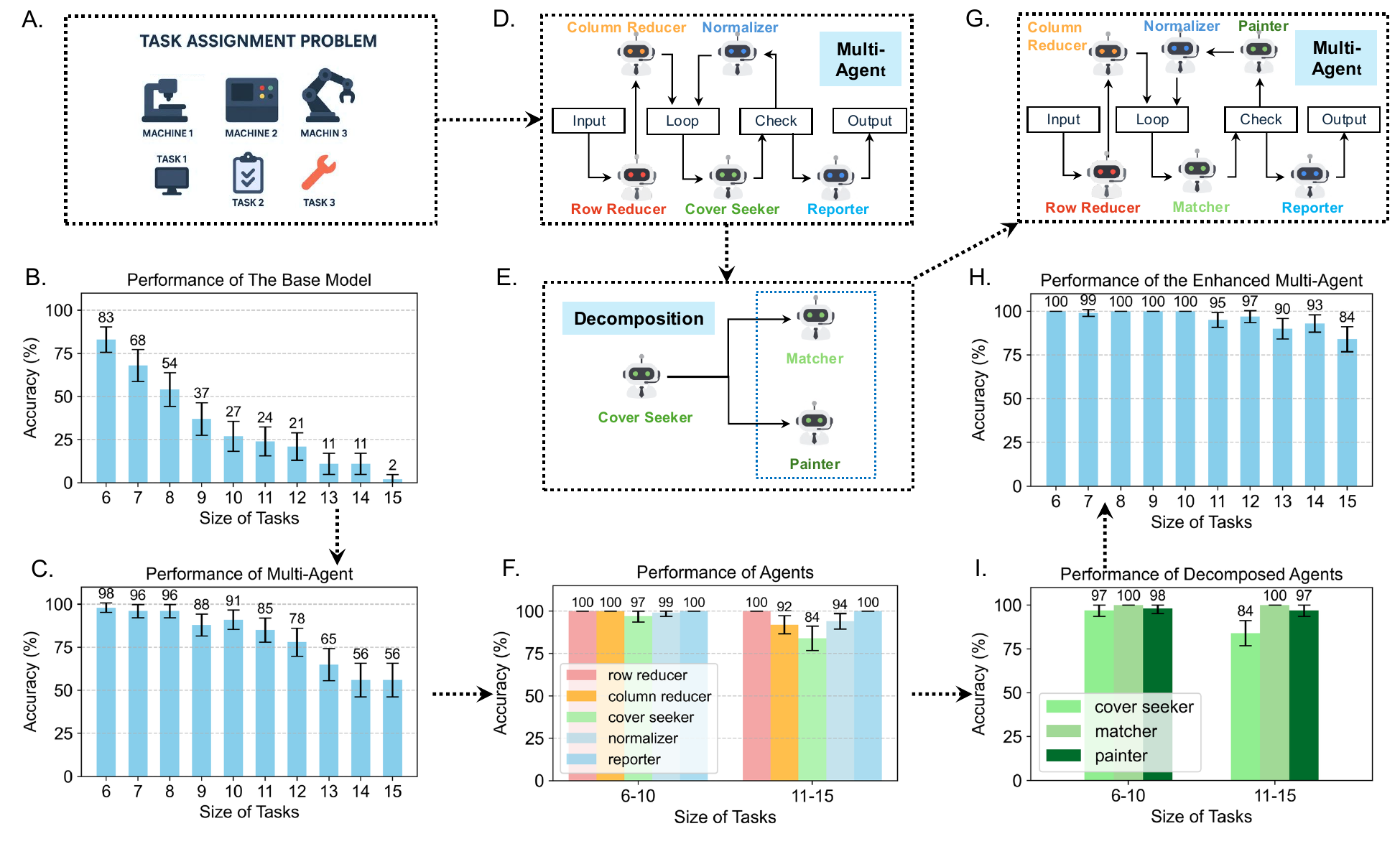}
    \caption{TAP evaluation of the KtR strategy. {\bf B}: Zero-shot accuracy of the baseline model.
{\bf D}: Initial blueprint derived from the Hungarian algorithm; its end-to-end accuracy is shown in {\bf C}.
{\bf F}: Per-agent accuracies within this blueprint, prompting the finer decomposition outlined in {\bf E}.
{\bf I}: Side-by-side comparison of per-agent accuracies before and after decomposition.
{\bf G}: Final, decomposed blueprint, whose overall accuracy appears in {\bf H}.}
    \label{fig:tap_mas}
\end{figure*}

{\bf Single LLM performance}.
Again, we evaluate the baseline performance of using o3-mini as a single agent. The o3-mini model achieves a relatively high performance (83\%) but decays quickly as in Figure \ref{fig:tap_mas}B: 37\% on problems of size 9, 21\% on problems of size 12 and finally is reduced to 3\% for problems of size 15.

{\bf Agent performance and further decomposition}.
Guided by the Hungarian algorithm \cite{kuhn1955hungarian}, our first KtR blueprint mapped each step to a single agent; Step 3 from Appendix \ref{sec: TAP problem solution} relied on a lone Cover Seeker rather than the later ``Matcher + Painter'' pair. This baseline already scored 98 \% (size 6), 88 \% (size 9), 78 \% (size 12), and 56 \% (size 15), validating the approach.

We then stress-tested each agent on two bands, with matrix sizes 6-10 and 11-15, to pinpoint weaknesses. One-shot agents were flawless: Row Reducer and Reporter reached 100 \% on both bands, and Column Reducer hit 100 \% / 92 \%. Normalizer held 99 \% / 94 \%, but Cover Seeker fell to 97 \% / 84 \%. Because Zero Seeker operates inside the main loop, its errors accumulate, making it the clear bottleneck for larger TAP instances.

We then perform a further decomposition of Step 3 in Section \ref{sec: KSP solution} in a two-step process: Step 3.1. Finding a {\bf maximal} collection of zero-entries, such that no two share a same row or column; Step 3.2. Finding a {\bf minimal} collection of rows and columns covering all zero-entries.

We believe this decomposition is helpful due to the following reasons. First, by a mathematical argument, the size of collections from sub-tasks 3.1 and 3.2 match. Second, a heuristic argument indicating that knowing the maximal collection of zeroes simplifies the task to find minimal collection of rows and columns. Last, the original Step 5 can then simply use the positions of the zeroes from Step 3.1, once optimal check passes. Note, this also explains why we prefer a further decomposition rather than fine-tuning the original agent, as a further decomposition improves the system flow as well.

Empirical pays off, as shown in Figure \ref{fig:tap_mas}I, Matcher reaches 100 \% accuracy on both difficulty bands, while Painter climbs to 98 \% and 97 \%—a sharp jump from the original Cover Seeker’s 97 \% and 84 \%. 

{\bf Final MAS performance}.
Leveraging the refined decomposition, we deploy a six-agent system (Figure \ref{fig:tap_mas}H) that solves size 6–10 instances almost flawlessly: almost 100 \% accuracy versus 83 – 27 \% for o3-mini zero-shot. It sustains high performance on size-11–15 tasks (95 \%, 97 \%, 90 \%, 93 \%, 84 \%); even the dip at size 15 far exceeds the 3 \% zero-shot baseline, highlighting the substantial capacity gain of our multi-agent architecture.

\section{Discussion}

Across two case studies, we show that disciplined decomposition combined with diagnosis-driven augmentation can raise low-capacity LLMs into dependable problem-solvers—and the gains grow further as the underlying model improves.

{\bf Proof-of-Concept on KSP}.
Baseline GPT-4o-mini accuracy collapsed as instance size increased, and a naïve three-agent blueprint offered little relief. Agent-level profiling singled out Trimmer as the lone bottleneck. Fine-tuning that agent alone—1 200 step-by-step examples, no changes elsewhere—lifted end-to-end accuracy from $\leq$ 18 \% to $\geq$ 70 \% across sizes 3–8, peaking at 95 \% on size 3. These results validate KtR’s “identify $\to$ isolate $\to$ augment” cycle: a targeted upgrade can rescue a small general-purpose model when both single-agent and untargeted multi-agent baselines fail.

{\bf Proof-of-Scalability on TAP}. Repeating the procedure with the stronger o3-mini model, we began with a five-agent design. Direct testing pinpointed a composite planning task that capped overall performance. KtR prescribes recursive refinement: we split that task into two simpler, typed sub-tasks, each amenable to the base model. Accuracy on the difficult sizes (11-15) rose from as low as 11 \% single-agent to $\geq$ 84 \% multi-agent, while sizes 6-10 reached 100 \%. Thus, as model capacity scales, KtR continues to amplify it rather than saturate.

Beyond these demonstrations, we would like to also outline several research threads that would deepen the approach in the future: 

$\bullet$ {\bf Model Portfolio Allocation.} Many leaves sit well below the flagship model’s frontier; mixing lighter or domain-tuned LLMs would slash cost while keeping accuracy, letting the controller pick ``just-enough'' capacity per task. 

$\bullet$ {\bf Complexity–Capacity Estimation.} Replace rule-of-thumb splits with a principled score that (i) quantifies task difficulty and (ii) predicts an LLM’s post-augmentation capacity, enabling data-driven decompositions.


$\bullet$ {\bf End-to-End Automation}. With the above metrics, the ultimate goal is to design a system that automates the whole KtR methodology: decomposing the tasks, evaluate capacities, and assemble the multi-agent system for solution.

\section{Limitations}

Despite demonstrating sizable performance gains, our study has several limitations that should guide future work.

{\bf Narrow task scope.} We evaluate KtR on two canonical optimization problems (Knapsack and Task-Assignment). While chosen to illustrate proof-of-concept and proof-of-scalability, these tasks have well-structured objective functions and small action spaces; generalizing to open-domain reasoning or multi-modal settings remains untested.

{\bf Synthetic data \& idealized inputs.} All problem instances are randomly generated and fully specified. Real-world inputs (noisy, partially observed, or adversarial) could degrade both decomposition quality and agent reliability.

{\bf Cost and latency trade-offs.} Although the per-agent inference cost is trending downward, we do not quantify absolute wall-clock latency, energy consumption, or controller overhead for large agent swarms.

{\bf Bottleneck identification heuristic.} We locate bottleneck subtasks via held-out accuracy screens; this assumes the availability of inexpensive ground-truth labels. Automated bottleneck detection without labels is an open problem.

\section{Ethical considerations}
Amplifying decision-making power by scaling agent crowds may exacerbate existing biases present in the base models; we have not run a bias or fairness audit.
Addressing these limitations—particularly expanding to less structured domains and benchmarking real-world resource usage—will be essential to establish the broader utility and safety of the KtR framework.

\appendix

\section{Appendix: Weighted No Free Lunch Theorem}

In this section, we present a weighted version of the No-Free-Lunch theorem. As the motivation, current approaches to MAS design can often result in overly general solutions that may exhibit suboptimal performance on specific and complex tasks. This sub-optimality arises partially from a lack of domain-specific inductive bias. To formalize this, we present a weighted variant of the No Free Lunch (NFL) theorem. The following demonstration, leveraging a weighted variant of the No Free Lunch theorem, quantitatively illustrates how inductive bias tailored to the target domain enhances performance.

That is, we present a formal proof showing that, under a non-uniform prior concentrated on a problem-specific subset of functions, a specialized learning algorithm achieves strictly lower expected risk than a general-purpose algorithm. 

Note the No-Free-Lunch theorem has been known to research community for more than two decades. Here what we present is a modification of the standard statement to better fit for our discussion on the MAS. As we don't find in literature the precise version of the NFL theorem as we stated below, we also present a proof for self-containedness. We do not claim any originality of the theorem and the proof.

\begin{theorem}[Weighted NFL]
Let $X$ be a finite input domain, $Y$ a finite label set, and 
$\mathcal{F}=Y^X$ the set of all functions $f\colon X\to Y$.  Consider
\begin{itemize}
  \item a general algorithm $A_0$ with constant expected loss $\varepsilon_0$ on every $f\in\mathcal{F}$,
  \item a specialized algorithm $A'$ satisfying
    \[
      L(h_{A'},f)\le
      \begin{cases}
        \varepsilon_1, & f\in\mathcal{F}',\\
        \varepsilon_2, & f\notin\mathcal{F}',
      \end{cases}
    \]
    where $\varepsilon_1<\varepsilon_0<\varepsilon_2$, and
  \item a prior $P$ with $P(f\in\mathcal{F}')=p$ and $P(f\notin\mathcal{F}')=1-p$.
\end{itemize}
If
\[
  p > \frac{\varepsilon_0-\varepsilon_2}{\varepsilon_1-\varepsilon_2},
\]
then the expected risk of $A'$ is strictly lower than that of $A_0$, i.e.
\[
  R(A') < R(A_0).
\]
\end{theorem}

\begin{proof}
By definition,
\[
  R(A_0) = \mathbb{E}_{f\sim P}[L(h_{A_0},f)] = \varepsilon_0
\]
and
\[
  R(A') = \mathbb{E}_{f\sim P}[L(h_{A'},f)] = p\,\varepsilon_1 + (1-p)\,\varepsilon_2.
\]
Hence
\begin{equation*}
	\begin{aligned}
		R(A') < R(A_0) \Longleftrightarrow & p\,\varepsilon_1 + (1-p)\,\varepsilon_2 < \varepsilon_0\\
		\Longleftrightarrow & p(\varepsilon_1 - \varepsilon_2) > \varepsilon_0 - \varepsilon_2,
	\end{aligned}
\end{equation*}

which rearranges to
\[
  p > \frac{\varepsilon_0 - \varepsilon_2}{\varepsilon_1 - \varepsilon_2}.
\]
This completes the proof.

\end{proof}

\section{Appendix: KSP and TAP description}\label{app: KSP and TAP}
In this appendix, we provide details about the KSP and TAP, including their problem description and algorithm based on which we design our MAS.

\subsection{KSP Problem Formulation}
The usual input of KSP involves a set of $N$ items, whose items are characterized by pairs $(w_i,v_i)$ of weights $w_i$ and values $v_i$, as well as a capacity value $W$. The goal of KSP is to find a subset of items such that the total weight does not exceed the given capacity while the total value is maximized. Mathematically, we record information of items by two vectors, both of dimension $N$: a weight vector $\vec{w}=(w_1,\cdots,w_N)$ and a value vector $\vec{v} = (v_1,\cdots,v_N)$. We also introduce the set of state-vectors $\{0,1\}^N$, whose elements are vectors $\vec{x}=(x_1,\cdots, x_N)$ where entries $x_i$ takes values between $0$ and $1$, indicating whether an item is chosen in a subset or not:
\begin{equation*}
    x_i = \begin{cases}
        1 & \text{item }i\text{ is chosen}\\
        0 & \text{item }i\text{ is excluded}
    \end{cases}
\end{equation*}

Thus state vectors controls which items is in the chosen subset, and the inner product of $\vec{x}$ with $\vec{w}$ and $\vec{v}$ then compute the total weight and total value for the given subset, respectively.

Given a weight vectors $\vec{w}$, a value vector $\vec{v}$, and the capacity constraint $W$, the objective of the Knapsack problem then can be formulated as finding the following (optimal) value:
\begin{equation}\label{eq:knapsack-obj}
    Z = \max_{\substack{\vec{x}\in \{0,1\}^N \\ \vec{x}\cdot \vec{w}\leq W}} \vec{x}\cdot \vec{v}.
\end{equation}

Here the maximal value is taken over all state vectors (or equivalently, all subsets of items) satisfying the constraint that the total weight $\vec{x}\cdot \vec{w}$ not exceeding the capacity $W$.

This version of the problem, where each item can either be fully included or not at all, is commonly known as the 0/1 KSP.

\subsection{KSP Problem Solution}\label{sec: KSP solution}

A classic approach to the Knapsack Problem iteratively enumerates all feasible states—a dynamic-programming strategy first introduced by Bellman \cite{bellman1957dynamic}. A feasible state can be defined by a pair $(current\_weight, current\_value)$ representing the accumulated weight and value of a set of items selected so far, such that $current\_weight \leq W$. We can describe the algorithm in the form of mathematical induction. We start with the initial set of feasible states $S_0 = \{(0,0)\}$, representing an empty set of chosen items. We then add items in to form a set $S_{k}$ from $S_{k-1}$ inductively, with capacity being aware: for each $k$, assuming that $S_{k-1}$ has been constructed, then we add the pair $(w_k,v_k)$ to all items in $S_{k-1}$ to form a new set $S_{add}$:
\[
S_{add} = \{(w+w_k,v+v_k)~{\rm for~all~}(w,v)\in S_{k-1}\}.
\]
Then, we trim the set according to the capacity:
\[
S_{trimmed} = \{(w,v)\in S_{add}~|~w\leq W\}.
\]
Note this also removes all repetitive states in the set. 
Finally we take the union of the two intermediate sets to create $S_k$:
\[
S_k = S_{k-1}\cup S_{trimmed}.
\]

The inductive step terminate when we have run through all items and obtaining the final set $S_N$, we pick the element in $S_N$ with maximal value, as the solution to the KSP. Explicitly,
\[
Z = \max_{(w,v)\in S_N} v
\]

\subsection{TAP Problem Formulation}
TAP seeks to optimally assign a set of $N$ resources (agents or workers) to $N$ tasks, where each potential assignment incurs a specific cost. With the constraint that each resource can only be assigned to one unique task, the objective of TAP is to find an assignment covering all tasks that minimizes the total cost. The resource-task specific cost is recorded in an $N \times N$ matrix $C$, where the entry $C_{ij}$ represents the cost associated with assigning resource $i$ to task $j$, for $i, j \in \{1, 2, \ldots, N\}$.

To formally define the problem, we introduce a set $\mathfrak{S}_n$ which can be described in either one of the following three equivalent ways:

$\bullet$ The group of automorphisms of the set $\underline{N} = \{1,2,\cdots, N\}$.

$\bullet$ The set (or group) of bijections from the set $\underline{N}$ to itself.

$\bullet$ The set of all permutations involving $N$ elements.

Note elements in $\mathfrak{S}_N$ convey the idea that each resource is assigned to a unique task.

Now, given the $N\times N$ cost matrix $C$, the objective of the TAP is to find the following (optimized) value
\[
Z = \max_{\sigma\in\mathfrak{S}_N}\left(\sum_{i=1}^NC_{i\sigma(i)}\right).
\]

Note when we treat $\sigma\in \mathfrak{S}_N$ as a (bijective) map from $\underline{N} = \{1,2,\cdots, N\}$ to itself, the notation $C_{i\sigma(i)}$ represents the entry on the $i$-th row and $\sigma(i)$-th column of the cost matrix $C$.

\subsection{TAP Problem Solution}\label{sec: TAP problem solution}

The typical solution for TAP is using the Hungarian algorithm \cite{kuhn1955hungarian}, which provides a polynomial-time method to find the objective value $Z$. We summarize the algorithm as follows.

\textbf{Step 1. Row Reduction.} For each row, we find the minimal element in the row and subtract it from all entries in the row, creating at least one zero on each row. Mathematically, starting from the original cost matrix $C_{N\times N}$, we create a new reduced matrix $C'$ such that for $i,j\in\{1,2,\cdots,N\}$, we have
\[
C'_{ij} = C_{ij} - \min_{1\leq k\leq N} C_{ik}
\]

\textbf{Step 2. Column Reduction.} Similarly, we further reduce $C'$ to $C''$ as follows: For each column, we find the minimal element in the column and subtract it from all entries in the column, guaranteeing at least one zero on each column. Mathematically, take
\[
C''_{ij} = C'_{ij} - \min_{1\leq k\leq N} C_{kj}
\]

\textbf{Step 3. Find covering lines.} We then find a smallest collection of rows and columns to cover all zeroes. Here smallest is the sense of the number of elements in the collection (of rows and column), over all possible such collections. If the size of this minimal collection, denoted by $L$, coincides with $N$, the size of the problem, then we skip Step 4 to enter the final stage of the algorithm.

\textbf{Step 4. Matrix Improvement.} However, if $L < N$, we need an improvement for the matrix $C''$ before looping back to Step 3: given the minimal collection of rows and columns from Step 3, we find the minimal value of all entries that are not covered, and then subtract this minimal value from all uncovered entries of $C''$, and then add this minimal value to all entries of $C''$ that are covered twice, i.e., by both rows and columns. Let $C'''$ be the resulting matrix.

\textbf{Step 5. Assignment Identification}. Once the condition $L=N$ is met, the final step is to identify the optimal assignment. This involves selecting a set of $N$ independent zeros from the current matrix $C'''$, such that no two selected zeros share the same row or column. Each selected zero at position $(i,j)$ corresponds to assigning agent $i$ to task $j$. The total cost of this optimal assignment is then calculated by summing the costs from the original cost matrix $C$ corresponding to these selected zero positions.

A non-trivial fact guaranteed by the Hungarian algorithm is that, in Step 5 the collection of zeroes might not be unique, while different collections are deemed to result in the same total summation of corresponding entries in the original cost matrix $C$.

\section{Ground-truth Data Preparation}\label{app: ground-truth}

We utilize Google OR-Tools \cite{ortools} to generate optimal solutions—serving as ground-truth datasets—for both problem scenarios. OR-Tools is a widely adopted open-source software suite developed by Google for solving combinatorial optimization problems. It is renowned for its efficiency and reliability in addressing NP-hard challenges through advanced optimization algorithms. The suite is distributed under the permissive Apache License 2.0, allowing unrestricted use, modification, and distribution \cite{ortools}.

For KSP, we generate random instances by assigning weights and values to items along with a maximum capacity constraint. Optimal solutions are then computed using OR-Tools’ dynamic programming approach. For TAP, we similarly generate random cost matrices that represent the cost of assigning workers to tasks. Optimal assignments are obtained using the Hungarian algorithm as implemented in OR-Tools, which efficiently minimizes the total assignment cost.

\section{Appendix: Prompt gallery}\label{app: prompts}
Note that all prompts we presented in the following, except for the self-check prompt for Trimmer Agent in KSP problem, are the system prompt for agents. The user prompt will only contain the precise problem to be handled by the agent in the form specified by the prompt.

\subsection{KSP prompts}
\subsubsection{Prompt for zero shot}
\begin{lstlisting}
You are an expert in the field of Knapsack Problem.

You are given a Knapsack Problem in the json format, of the following form:
{
    "id" : str,
    "items" : list of pairs of integers,
    "capacity" : int
}

Each pair in the list is a pair of integers of the form [weight, value], i.e., the first entry is the weight and the second entry is the value.

Your task is to solve the Knapsack Problem and provide the optimal solution. That is, you need to find a subset of the pairs that maximizes the total value, subject to the constraint that the total weight of the subset is less than or equal to the capacity.

Please think step by step when solving the problem.

You need to return the optimal solution in the following json format:
{
    "max_value" : int,
}
Please only return the json format, nothing else.
\end{lstlisting}

\subsubsection{Prompt for Worker Agent}
\begin{lstlisting}
You are a key member of a multi-agent team collaboratively solving the Knapsack Problem. Your specific role is the Worker, responsible for performing mathematical computations for the team.

You will receive input in the following JSON format:
{"c_list": [[int, int], ...], "s_item": [int, int]}
Each pair within 'c_list' contains two integers.

Your task is to:
- Add 's_item' to each pair in 'c_list' entry-wise. For instance, if a pair in 'c_list' is '[2, 5]' and 's_item' is '[3, 4]', the result should be '[2+3, 5+4] = [5, 9]'.

To ensure accuracy:
- Proceed systematically, applying step-by-step reasoning.
- Carefully perform each addition individually for all pairs provided in the list.

Your response must strictly follow this JSON format:
{"n_list": [[int, int], ...]}

Return only the specified JSON object without any additional commentary or text.
\end{lstlisting}

\subsubsection{Prompt for Trimmer Agent}
\begin{lstlisting}
You are a key member of a multi-agent team collaboratively solving the Knapsack Problem. Your specific role is the Trimmer, responsible for trimming the list based on the given capacity constraint.

You will receive input in the following JSON format:
{"n_list": [[int, int], ...], "capacity": int}

Each pair within 'n_list' contains two integers: the first integer represents the weight, and the second integer represents the value.

Your task is to:
- Carefully analyze each pair in the provided list.
- Remove all pairs whose weight (the first integer) strictly exceeds the specified capacity.
- If identical pairs appear multiple times, retain only one instance of each.

To ensure accuracy:
- Proceed systematically, applying step-by-step reasoning.
- Verify each pair carefully against the capacity constraint.

Your response must strictly follow this JSON format:
{"t_list": [[int, int], ...]}

Return only the specified JSON object without any additional commentary or text.
\end{lstlisting}

\subsubsection{Prompt for Reporter Agent}
\begin{lstlisting}
You are a key member of a multi-agent team collaboratively solving the Knapsack Problem. Your specific role is the Reporter, responsible for determining and clearly reporting the final answer based on the provided information.

You will receive input in the following JSON format:
{"c_list": [[int, int], ...]}
Each pair within 'c_list' contains two integers: the first integer represents the weight, and the second integer represents the value.

Your task is to carefully analyze this list, identify the pair with the maximal value (the second integer in each pair), and report only that maximal value. If the list is empty, then report the maximal value as 0.

To ensure accuracy:
- Proceed systematically, applying step-by-step reasoning.
- Carefully examine every pair in the provided list.

Your response must strictly follow this JSON format:
{"max_value": int}

Return only the JSON object as specified above, without any additional commentary or text.
\end{lstlisting}

\subsubsection{Self-check prompt for Trimmer Agent}
\begin{lstlisting}
To better fulfill your task, please conduct a double check on the result you just provided. If your answer is already correct, please confirm by copying the last output.

When double check, please pay attention to the following typical types of mistakes:

In particular, please check if you made any typical mistakes as listed below:
1. If you added in a pair that is not in the original n_list.
2. If there is still a pair in the t_list that still exceeds the capacity.
3. If there is a pair in n_list that does not exceed the capacity but is not in the t_list.

If you found any errors, please create a corrected answer.

In either case, please follow the format requirement of the output.
\end{lstlisting}

\subsection{TAP prompts}
\subsubsection{Prompt for zero shot}
\begin{lstlisting}
You are an expert in solving the Assignment Problem. In the assignment problem, there are n workers and n jobs. Each worker has a cost of assigning to each job. Each worker can only be assigned to one job. Your task is to find the optimal assignment of workers to jobs that minimizes the total cost.

You are given the problem in the following json format:

{
    "id" : str,
    "cost_matrix" : list of lists of integers
}

The cost matrix is a square matrix of size n x n, where n is the number of workers and jobs, in the form of a nested list [[int, int, ...], [int, int, ...], ...]. The (i, j)th entry of the matrix represents the cost of assigning the ith worker to the jth job.

Your task is to find the optimal assignment of workers to jobs that minimizes the total cost.

Please think step by step when solving the problem.

You need to return the optimal assignment in the following json format:

{
    "optimal_cost" : int
}

Please only return the json format, nothing else.
\end{lstlisting}

\subsubsection{Prompt for Row Reducer Agent}
\begin{lstlisting}
You are given a matrix in the following json format:

{
    "matrix" : list of lists of integers
}

The matrix is in the form of a nested list [[int, int, ...], [int, int, ...], ...].

Your task is to reduce the matrix by subtracting the minimum value of each row from all the elements in that row.

Please think step by step when solving the problem:
Step 0: Work on one row at a time.
Step 1: Find the minimum value of the row.
Step 2: Subtract the minimum value of the row from all the elements in that row.
Step 3: Return the reduced matrix in the following json format:

{"reduced_matrix" : list of lists of integers}

Please only return the json format, nothing else.
\end{lstlisting}

\subsubsection{Prompt for Column Reducer Agent}
\begin{lstlisting}
You are given a matrix in the following json format:

{
    "matrix" : list of lists of integers
}

The matrix is in the form of a nested list [[int, int, ...], [int, int, ...], ...].

Your task is to reduce the matrix by subtracting the minimum value of each column from all the elements in that column.

Please think step by step when solving the problem:
Step 0: Work on one column at a time.
Step 1: Find the minimum value of the column.
Step 2: Subtract the minimum value of the column from all the elements in that column.
Step 3: Return the reduced matrix in the following json format:

{"reduced_matrix" : list of lists of integers}

Please only return the json format, nothing else.
\end{lstlisting}

\subsubsection{Prompt for Cover Seeker Agent}
\begin{lstlisting}
You are given a problem in the following json format:

{
    "matrix" : list of lists of integers
}

The matrix is in the form of a nested list [[int, int, ...], [int, int, ...], ...].

Your task is to find a smallest collection of rows and columns of the matrix, such that any zeroes in the matrix is contained in a chosen row or column. Small means the sum of the sizes of the row and column collections is the smallest possible.

Please think step by step when solving the problem, and return your response in the following json format:

{"collum_collection" : [int, int, ...], "row_collection" : [int, int, ...]}

The integers in the collum_collection and row_collection are the indices of the rows and columns that you choose.

Please only return the json format, nothing else.
\end{lstlisting}

\subsubsection{Prompt for Matcher Agent}
\begin{lstlisting}
You are given a matrix in the following json format:

{
    "matrix" : list of lists of integers
}

The matrix is in the form of a nested list [[int, int, ...], [int, int, ...], ...].

Your task is to find the largest collection of zeroes in the matrix, such that no two zeroes are in the same row or column.

Please think step by step when solving the problem, and return your response in the following json format:

{"largest_collection" : [[int, int], [int, int], ...]}

The list of pairs of integers is in the form of [[row_index, column_index], [row_index, column_index], ...].

Please only return the json format, nothing else.
\end{lstlisting}

\subsubsection{Prompt for Painter Agent}
\begin{lstlisting}
You are given a problem in the following json format:

{
    "matrix" : list of lists of integers
    "collection" : list of lists of integers
}

The matrix is in the form of a nested list [[int, int, ...], [int, int, ...], ...].

The collection is in the form of a nested list [[int, int], [int, int], ...].

Your task is to find a smallest collection of rows and columns of the matrix, such that any zeroes in the matrix is contained in a chosen row or column. Small means the sum of the sizes of the row and column collections is the smallest possible.

To assist you, you are provided with a collection of zeroes in the input json format. The collection contains the positions of a maximal collection of zeroes in the matrix, such that no two zeroes are in the same row or column.

Please use this collection of zeroes to find the rows and columns as desired. More precisely, you should first choose one row or column for each zero in the collection, such that the chosen rows and columns cover as much of the zeroes in the matrix as possible. Then add in more rows or columns if needed.

Please think step by step when solving the problem, and return your response in the following json format:

{"collum_collection" : [int, int, ...], "row_collection" : [int, int, ...]}

The integers in the collum_collection and row_collection are the indices of the rows and columns that you choose.

Please only return the json format, nothing else.

\end{lstlisting}

\subsubsection{Prompt for Normalizer Agent}
\begin{lstlisting}
You are given a problem in the following json format:

{
    "matrix" : list of lists of integers
    "collumn_collection" : list of integers
    "row_collection" : list of integers
}

The matrix is in the form of a nested list [[int, int, ...], [int, int, ...], ...].
The collumn_collection and row_collection are the indices of some selected rows and columns that covers all the zeroes in the matrix.

Your task is the following:
1. Find the minimal value in the matrix that is not covered by the selected rows and columns.
2. If this value is 0, return the original matrix.
3. If this value is not 0, subtract this value from all uncovered entries in the matrix.
4. For the entries that covered by both a selected row and a selected column, add this value to the entries.
5. For the entries that are covered by a selected row or column, but not both, do nothing.
6. Please return the updated matrix in the following json format:

{"normalized_matrix" : list of lists of integers}

Please only return the json format, nothing else.
\end{lstlisting}

\subsubsection{Prompt for Reporter Agent}
\begin{lstlisting}
You are given a problem in the following json format:

{
    "matrix" : list of lists of integers
    "collection" : list of lists of integers
}

The matrix is in the form of a nested list [[int, int, ...], [int, int, ...], ...].
The collection contains a set of entries of the matrix in the form of [[row_index, column_index], [row_index, column_index], ...].

Your task is the following:
1. Sum up the values of all the entries in the collection.
2. Return the total value in the following json format:

{"total_value" : int}

Please only return the json format, nothing else.
\end{lstlisting}

\bibliography{latex/custom}

\end{document}